\documentclass[a4paper]{styles/svproc}
\usepackage{url}
\usepackage{graphicx}
\usepackage{subcaption}
\usepackage{amsmath}
\usepackage{amssymb}
\usepackage{tikz}
\usetikzlibrary{
    arrows.meta,
    positioning,
    calc
}
\usepackage{tikz-cd}

\begin{document}
\mainmatter              
\title{Complete Motion Planning using Workspace-Fibered Decomposition for $n$R-Planar Manipulator}
\titlerunning{Workspace Fibration}  
%
\author{Aayush Rath \and Antony Thomas}

\authorrunning{Aayush Rath and Antony Thomas}
%

\institute{Robotics Research Center, IIIT Hyderabad, Hyderabad 500032, India.\\
\email{anusandhaan1@gmail.com,antony.thomas@iiit.ac.in}}

\maketitle              

\begin{abstract}
We propose a workspace-fibered decomposition framework for motion planning in $n$R planar redundant manipulators operating in cluttered environments. Rather than planning directly in the full $n$-dimensional configuration space, the method incrementally constructs obstacle-constrained reachable workspaces of lower-dimensional non- redundant subchains and recursively lifts them through redundant orientation fibers. This yields a sequence of reduced planning manifolds that preserve branch-consistent reachability structure while avoiding explicit construction of the full configuration-space obstacle geometry.

We first establish that, for planar position-only manipulators, the obstacle -constrained reachable workspace induced by the minimal non-redundant subchain provides an exact characterization of feasibility with respect to the connected component of the start configuration, enabling early infeasibility detection prior to introducing redundant degrees of freedom (DOF). We then introduce an incremental fiber-lifting procedure that propagates reachable workspace structure through successive redundant links while enforcing local inverse-kinematic branch consistency using Jacobian determinant continuity constraints.

The resulting representation admits efficient reduced-space planning directly on recursively-constructed workspace-fiber manifolds. Experimental results on redundant $n$R planar manipulators demonstrate that the proposed construction preserves collision-free connectivity structure across successive lifting stages while substantially reducing collision checking complexity relative to direct configuration space reasoning.

\end{abstract}

\section{Introduction}
\label{sec:introduction}
Motion planning has been an active area of research for nearly 50 years~\cite{orthey2023ARCRAS}. Nevertheless, complete motion planning, that is, finding a collision-free path when one exists and certifying infeasibility otherwise, remains a challenging problem~\cite{li2023IJRR}. Sampling-based motion planners (SBMPs)~\cite{karaman2011IJRR} can efficiently discover feasible paths in high-dimensional configuration spaces. However, when no collision-free path exists, SBMPs are unable to certify infeasibility and typically continue searching until a termination condition or timeout is reached. Existing complete planning methods are generally limited to low-dimensional configuration spaces~\cite{zhang2008IJRR}, since their computational complexity scales poorly with the dimensionality of the planning problem.

Recent approaches have enabled infeasibility certification, but their scalability remains limited, with demonstrations restricted to configuration spaces of up to 5-DOF~\cite{li2023RAL,thomas2025arxiv}. To address the scalability limitations of complete planning, several recent works have explored hierarchical abstractions and lower-dimensional relaxations of the configuration space. Quotient-space planning methods~\cite{orthey2018IROS,orthey2022ISRR} construct nested sequences of admissible lower-dimensional configuration spaces by progressively collapsing subsets of robot degrees of freedom. These reduced representations allow planning to proceed hierarchically, where feasibility discovered in lower-dimensional quotient spaces guides exploration in the full configuration space. Building upon this principle, multilevel sparse roadmap methods~\cite{orthey2021ICRA} recursively construct sparse graph structures across quotient-space hierarchies and demonstrate substantial improvements in planning efficiency for high-dimensional robotic systems.


In contrast, the approach proposed in this work constructs reduced representations through obstacle-constrained reachable workspace manifolds induced by lower-dimensional non-redundant subchains. Rather than recursively decomposing the configuration space directly, we incrementally propagate reachability structure through redundant orientation fibers. This yields a workspace-centered hierarchical representation that preserves branch-consistent reachability information while avoiding explicit construction of the full high-dimensional configuration-space obstacle geometry. Furthermore, by characterizing feasibility through the reachable workspace of non-redundant subchains, the proposed framework naturally admits early infeasibility detection prior to introducing additional redundant degrees of freedom.

\section{Related Work}
\label{sec:related_work}
The field of redundant manipulator motion planning has evolved from local velocity-based approaches toward global topological methods that characterize the entire solution space of the Inverse Kinematics Problem (IKP). For redundant manipulators, the IKP generally admits infinitely many solutions. For a regular\footnote{A regular point is a configuration of a non-redundant manipulator at which the Jacobian is full rank. A critical point is a configuration at which the Jacobian loses rank. The corresponding end-effector positions are called regular values and critical values, respectively.} end-effector location in the workspace, the set of solutions forms an r-dimensional submanifold of the configuration space, where r is the degree of redundancy. Depending on the end-effector position in the workspace, this solution set may be partitioned into multiple disjoint manifolds. Burdick defined each of the disjoint r-dimensional manifold of IKP as a Self Motion Manifold (SMM)~\cite{burdick1989ICRA}. He further showed that, for an $n$R-planar manipulator, the maximum number of SMMs is two.

Early foundational work \cite{Wenger1993} addressed the global feasibility of trajectories in cluttered environments by introducing (r+1)-dimensional feasibility maps. Their approach involves adding r extra parameters to the operational space, effectively treating the robot as non-redundant in an extended space to analyze joint limits and obstacle obstructions. While comprehensive, this method is generally suited for systems with low degrees of redundancy.
Furthering the topological perspective, \cite{Luck1997} developed a space discretization method based on self-motion topology. By partitioning the configuration space into c-bundles—regions where all SMMs are homotopic—\cite{Luck1997} established a bundle connectivity graph. This framework allows for hierarchical path planning that successfully avoids local minima and deadlocks by reasoning about the entire kinematic map.
In the context of workspace analysis, \cite{Peidro2018} proposed a sampling-based method to identify the collision-free workspace of redundant robots by monitoring the vanishing of SMMs. They demonstrated that internal motion barriers occur when connected components of these manifolds disappear due to kinematic constraints or link collisions. Their approach utilizes inverse kinematics to densely sample SMMs, followed by clustering and matching phases to detect these topological changes.
Most recently, \cite{FabregatJaen2025} presented a unified framework for global redundancy resolution that synthesizes topological analysis with numerical optimization. Their method improves upon previous sampling techniques through vectorized SMM generation and the construction of a Self-Motion Domain (SMD) graph. This framework provides a highly flexible pipeline: it first identifies feasible \textit{c-bundle chain} along a task trajectory and then employs constrained Quadratic Programming (QP) to optimize the resulting joint trajectories while strictly maintaining them on the SMMs. Unlike traditional local methods, this approach ensures global optimality and cyclic consistency even in high-dimensional joint spaces

In parallel to topological approaches for redundancy resolution, a substantial body of work has studied configuration-space decomposition methods for complete motion planning. Classical approaches such as vertical cell decomposition \cite{Alon2008ACM} and Cylindrical Algebraic Decomposition (CAD) \cite{Collins1975} partition the free configuration space into connected semi-algebraic cells that admit exact adjacency reasoning and completeness guarantees. These methods provide certificates of infeasibility whenever no collision-free path exists and form some of the earliest globally-complete motion planning frameworks.

To improve practical scalability, several adaptive decomposition strategies have been proposed. Star-shaped decomposition methods partition the collision-free configuration space into collections of star-shaped regions through adaptive subdivision and guard-point selection, enabling efficient local connectivity tests while preserving global completeness properties \cite{Varadhan2005RSS}. Similarly, the adaptive cell decomposition framework performs hierarchical subdivision of the configuration space and constructs exact connectivity graphs for collision-free planning. These approaches significantly reduce unnecessary subdivision in free regions while retaining rigorous completeness guarantees \cite{zhang2008IJRR}.

Despite their theoretical strengths, configuration-space decomposition methods generally scale poorly with dimensionality due to the exponential growth of the configuration space. Exact cell decomposition techniques rapidly become computationally intractable beyond low-dimensional systems, typically limiting their practical applicability to configuration spaces of dimension four or lower. For redundant manipulators, where the configuration space dimension grows directly with the number of joints, explicit decomposition of the full collision-free configuration space becomes prohibitively expensive.

More recently, quotient-space and multilevel planning methods have been introduced to improve scalability in high-dimensional motion planning by exploiting nested robot structure. Quotient-space decomposition methods construct a hierarchy of reduced configuration spaces by progressively collapsing subsets of degrees of freedom, enabling planning to proceed first in lower-dimensional quotient representations before lifting solutions into higher-dimensional spaces. Building upon this idea, the quotient-space planning \cite{orthey2018IROS} and later the Multilevel Sparse Roadmap \cite{orthey2021ICRA} framework, which recursively construct sparse roadmap representations across nested quotient spaces of increasing dimensionality. These methods significantly improve planning efficiency for high-dimensional robots by leveraging admissible lower-dimensional abstractions and hierarchical feasibility propagation. However, the quotient spaces are still defined directly in terms of configuration-space reductions and do not explicitly exploit the workspace topology induced by non-redundant subchains or the self-motion structure of redundant manipulators. In contrast, the proposed framework constructs reduced representations through obstacle-constrained reachable workspace manifolds and recursively propagates this structure through redundant orientation fibers, yielding a workspace-centered decomposition that preserves branch-consistent reachability information during lifting.

\section{Problem Statement}
\label{sec:problem_statement}
Consider a planar $n$R manipulator with configuration space
\[
\mathcal{C} = T^n = (SO(2))^n
\]
Let the forward kinematics of the subchain comprising the first $m$ links 
be denoted
\[
f_m : T^m \rightarrow \mathbb{R}^2,
\]
mapping joint configurations to the position of the $m$-th link 
end-effector in the plane.

Rather than planning directly in the high-dimensional space $T^n$, we 
seek an incremental decomposition that builds the reachable planning 
domain one fiber at a time, starting from the minimal non-redundant 
subchain and extending by one redundant degree of freedom at each step. 
The resulting decomposition is designed to:
\begin{enumerate}
    \item Preserve the topological structure of the obstacle-constrained 
    reachable workspace at each stage,
    \item Enable early detection of infeasibility before the full     configuration space is constructed,
    \item Reduce the cost of collision checking by inheriting free-space 
    validity from previous stages.
\end{enumerate}

We now define the problem studied in this work. \textit{Given a start configuration $q_s \in T^n$ and a target end-effector 
position $x_g \in \mathbb{R}^2$, we seek a continuous collision-free 
path
\[
\gamma : [0,1] \rightarrow \mathcal{C}_{\mathrm{free}} \subset T^n
\]
such that $\gamma(0) = q_s$ and $f_n(\gamma(1)) = x_g$. If no such path exists, the algorithm returns a certificate (proof) of infeasibility.}


\begin{figure}[t]
\centering

\resizebox{\columnwidth}{!}{
\begin{tikzpicture}[
    node distance=1.8cm,
    every node/.style={font=\scriptsize},
    manifold/.style={
        draw,
        thick,
        rounded corners=6pt,
        minimum width=2.4cm,
        minimum height=1.2cm,
        align=center,
        fill=blue!5
    },
    fiber/.style={
        draw,
        thick,
        circle,
        minimum size=0.8cm,
        fill=green!10
    },
    arrow/.style={
        -{Latex[length=2mm]},
        thick
    }
][h]


\node[manifold] (cspace)
{
Full C-space\\
$\mathcal{C}=T^2$
};

\node[manifold,
right=of cspace]
(cfree)
{
Connected Free\\
Component\\
$\mathcal{C}^{(s)}_{\mathrm{free}}\subset T^2$
};

\node[manifold,
right=of cfree]
(workspace)
{
Reachable Workspace\\
$\mathcal{R}_{2R}\subset\mathbb{R}^2$
};

\node[fiber,
below=1.5cm of workspace]
(fiber)
{
$S^1$
};

\node[manifold,
right=of workspace]
(reduced)
{
Reduced Space\\
$\mathcal{R}_{2R}\times S^1$
};


\draw[arrow]
(cspace)
-- node[above]
{\scriptsize Restriction}
(cfree);

\draw[arrow]
(cfree)
-- node[above]
{\scriptsize FK}
(workspace);

\draw[arrow]
(workspace)
-- node[above]
{\scriptsize Extrusion}
(reduced);

\draw[arrow]
(fiber)
-- node[right]
{\scriptsize Fiber Lift}
(reduced);


\node
at ($(workspace)!0.5!(reduced)+(0,1.1)$)
{
$\left|\det(J_a)-\det(J_b)\right|<\epsilon$
};

\draw[dashed, thick]
($(workspace)!0.5!(reduced)+(0,0.7)$)
--
($(workspace)!0.5!(reduced)+(0,-0.1)$);

\end{tikzpicture}
}

\caption{
Topology-aware reduced-space planning framework.
A connected collision-free component of the 2R subchain is projected into an obstacle-constrained reachable workspace manifold.
An $S^1$ orientation fiber corresponding to the redundant terminal link is then lifted over the workspace manifold, yielding the reduced planning space $\mathcal{R}_{2R}\times S^1$.
Local inverse-kinematic branch continuity is approximately preserved through Jacobian determinant consistency constraints (see Section~\ref{subsec:determinant}).
}

\label{fig:fiber_lifting_framework}

\end{figure}
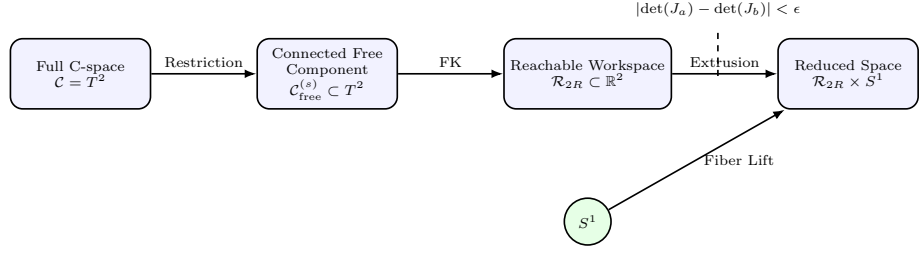

\section{Workspace Fibration}
\label{sec:workspacefibration}
\subsection{Overview}
The central idea is to construct the reachable planning domain 
incrementally. At each stage $m$, we maintain an obstacle-constrained 
reachable workspace $\mathcal{R}_{mR} \subset \mathbb{R}^2$ for the 
$m$-link subchain, and extend it to stage $m+1$ by lifting over a new 
fiber $S^1$ corresponding to the $(m+1)$-th redundant link. The 
incremental sequence is:

\[
\mathcal{R}_{2R}
\;\xrightarrow{\;\text{lift}\;}\;
\mathcal{R}_{2R} \times S^1
\;\xrightarrow{\;\text{project}\;}\;
\mathcal{R}_{3R}
\;\xrightarrow{\;\text{lift}\;}\;
\mathcal{R}_{3R} \times S^1
\;\xrightarrow{\;\text{project}\;}\;
\mathcal{R}_{4R}
\;\rightarrow\; \cdots
\]

At each stage the planning domain is a product space 
$\mathcal{R}_{mR} \times S^1$, where $\mathcal{R}_{mR}$ is the 
obstacle-constrained reachable workspace of the $m$-link subchain and 
$S^1$ is the orientation fiber of the next redundant link. The full 
$n$R planning problem is solved on the final reduced space 
$\mathcal{R}_{(n-1)R} \times S^1$. Figure~\ref{fig:fiber_lifting_framework} provides a schematic illustration of the recursive reduced-space construction and the corresponding lifting procedure.

\subsection{Base Stage: Non-Redundant 2R Subchain}

The incremental construction is initialized at the minimal 
non-redundant subchain. For a planar arm with position-only task, this 
is the 2R subchain with forward kinematics
\[
h : T^2 \rightarrow \mathbb{R}^2
\]

Starting from a specified start configuration 
$q_s = (\theta_1^s, \theta_2^s) \in T^2$, we identify the connected 
free-space component
\[
\mathcal{C}^{(s)}_{\mathrm{free}} \subset T^2
\]
containing $q_s$ by building the collision occupancy of $T^2$ and 
applying connected component labeling with toroidal adjacency. The 
obstacle-constrained reachable workspace of the 2R subchain is then
\[
\mathcal{R}_{2R} = h\!\left(\mathcal{C}^{(s)}_{\mathrm{free}}\right) 
\subset \mathbb{R}^2
\]

Unlike the kinematic reachable workspace, $\mathcal{R}_{2R}$ inherits 
the topological constraints of the connected free-space component — in 
particular, obstacle-induced disconnections in $T^2$ propagate into 
$\mathcal{R}_{2R}$ as gaps or holes in the reachable workspace.

\textbf{Early infeasibility detection.} If the goal position 
$x_g \notin \mathcal{R}_{2R}$, the query is immediately infeasible 
from the given start configuration, regardless of how many additional 
redundant links are present. No further computation is required. This 
provides a computationally cheap necessary condition for feasibility 
that can be evaluated before any redundant link is considered.

\subsection{Jacobian-Induced Branch Structure}
\label{subsec:determinant}

\begin{figure*}[t!]
    \centering
    \begin{subfigure}{0.48\textwidth}
        \centering
        \includegraphics[width=0.9\linewidth]{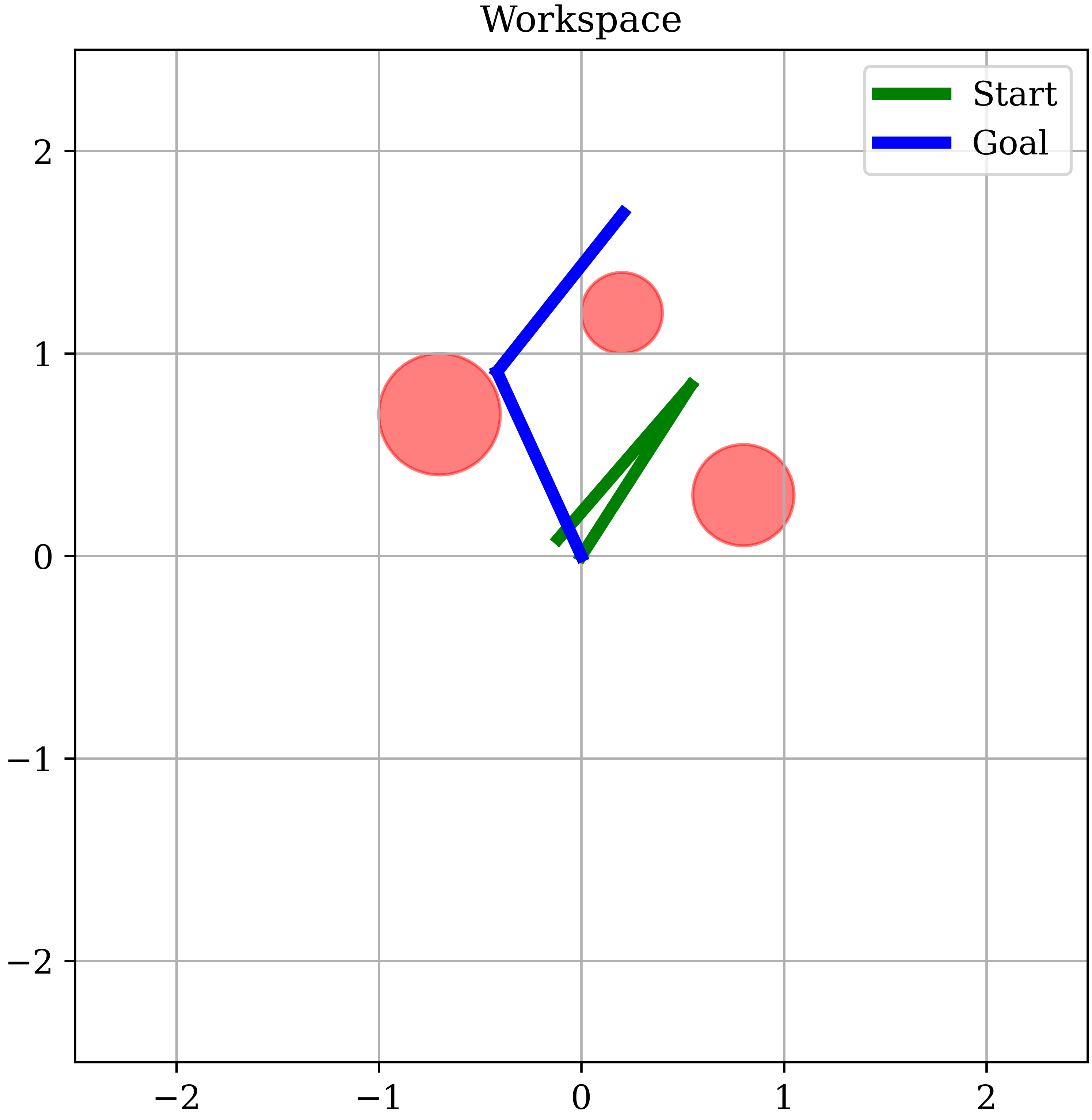}
        \caption{Environment setup with obstacles for 2R manipulator}
        \label{fig:workspace_detj}
    \end{subfigure}
    ~
    \begin{subfigure}{0.48\textwidth}
        \centering
        \includegraphics[width=1.1\linewidth]{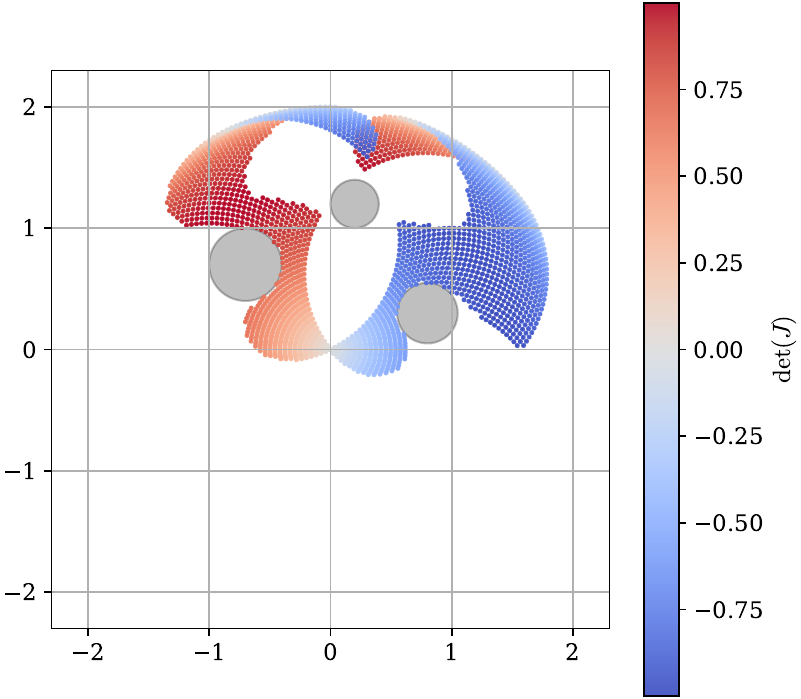}
        \caption{Reachable Workspace of the second link end-effector}
        \label{fig:workspace_detj}
    \end{subfigure}
    \caption{Reachable workspace for 3R planning}
    \label{rec}
\end{figure*}

The map $h$ is generically two-to-one over $\mathcal{R}_{2R}$ (except at singular configurations), since 
the 2R inverse kinematics admits at most two solutions corresponding to 
elbow-up and elbow-down configurations. Each workspace sample in 
$\mathcal{R}_{2R}$ retains its associated Jacobian determinant 
\[
\det(J) = l_1 l_2 \sin(\theta_2),
\]
which serves as a branch label for subsequent IK reconstruction. The 
sign of $\det(J)$ partitions $\mathcal{C}^{(s)}_{\mathrm{free}}$ into 
two open subsets, as can be seen from Figure~\ref{fig:workspace_detj}.
\[
\mathcal{C}^+ = \{q \mid \det(J(q)) > 0\}, \quad
\mathcal{C}^- = \{q \mid \det(J(q)) < 0\},
\]
separated by the singular set $\Sigma = \{q \mid \det(J(q)) = 0\}$.

\begin{proposition}
For a planar 2R subchain, the self-motion leaf 
$S_x = \{q \in T^2 \mid h(q) = x\}$ at a generic workspace point 
$x \in \mathcal{R}_{2R}$ has at most two connected components, 
corresponding to the elbow-up and elbow-down configurations.
\end{proposition}

\begin{proof}
The IK equation gives 
$\cos\theta_2 = (|x|^2 - l_1^2 - l_2^2)/(2l_1 l_2)$, admitting at 
most two solutions for $\theta_2 \in (-\pi,\pi]$ since cosine is 
injective on $(0,\pi)$ and solutions are related by 
$\theta_2 \mapsto -\theta_2$. Each $\theta_2$ uniquely determines 
$\theta_1$, giving at most two isolated points in $T^2$.
\end{proof}

\subsection{Incremental Fiber Lifting}

Given the reachable workspace $\mathcal{R}_{mR}$ of the $m$-link 
subchain, we extend to the $(m+1)$-link subchain by lifting over the 
orientation fiber of the $(m+1)$-th link. The reduced planning space 
at stage $m$ is
\[
\mathcal{M}^{(m)}_{\mathrm{red}} = \mathcal{R}_{mR} \times S^1
\]
where $S^1$ is discretized uniformly into $N_\theta$ samples over 
$\theta_{m+1} \in [-\pi, \pi)$.

For each workspace sample $x \in \mathcal{R}_{mR}$ and each 
$\theta_{m+1} \in S^1$, the collision-free validity of the full 
$(m+1)$-link configuration is evaluated by:
\begin{enumerate}
    \item Recovering the $m$-link IK solution consistent with the 
    stored $\det(J)$ branch label. Collision checking is unnecessary, since collision-free validity is inherited from $\mathcal{C}^{(s)}_{\mathrm{free}}$.
    \item Checking only the $(m+1)$-th link for collision against 
    obstacles.
\end{enumerate}

This produces the reduced occupancy of 
$\mathcal{M}^{(m)}_{\mathrm{red}}$. The reachable workspace of the 
$(m+1)$-link subchain is then obtained by projecting the free subset 
of $\mathcal{M}^{(m)}_{\mathrm{red}}$ back through the forward 
kinematics of the $(m+1)$-link subchain:
\[
\mathcal{R}_{(m+1)R} = f_{m+1}\!\left(
    \mathrm{free}\!\left(\mathcal{M}^{(m)}_{\mathrm{red}}\right)
\right)
\]

This projection feeds into the next stage of the incremental 
construction, replacing $\mathcal{R}_{mR}$ with the updated reachable 
workspace $\mathcal{R}_{(m+1)R}$ for the subsequent fiber lift.

\textbf{Stage-wise infeasibility detection.} At each stage $m$, if 
$x_g \notin \mathcal{R}_{mR}$, the query is infeasible from the 
current start component and the construction terminates early. This 
allows infeasibility to be detected incrementally, often well before 
the full $n$-link planning space is constructed.

\subsection{Branch-Consistent Connectivity}

Neighbor connectivity in $\mathcal{M}^{(m)}_{\mathrm{red}}$ enforces 
a local Jacobian continuity constraint. Adjacent lattice nodes $a$ and 
$b$ are connected only if 
\[
|\det(J_a) - \det(J_b)| < \varepsilon
\]
for a prescribed threshold $\varepsilon > 0$. This condition suppresses 
discontinuous IK branch switches while permitting smooth passage 
through singular configurations when they are collision-free, 
consistent with the physical traversability of singularities at the 
position level.

The effect of the Jacobian-based continuity criterion is illustrated in Figure~\ref{fig:workspace_detj}. Although the reachable workspace appears connected above the central obstacle, the determinant map on the right shows that this region belongs to different $\det(J)$ branches. The continuity constraint therefore prevents edges from connecting nodes across the branch boundary, eliminating spurious inverse-kinematics transitions. Consequently, the third link cannot cross directly through this region, and a valid path for the 3R subchain must instead pass through the lower region.

\subsection{Planning on the Final Reduced Manifold}

For an $n$R planar arm, the incremental construction terminates at 
stage $m = n-1$, yielding the final reduced planning space
\[
\mathcal{M}_{\mathrm{red}} = \mathcal{R}_{(n-1)R} \times S^1
\]

Path planning is performed directly on the reduced occupancy lattice 
of $\mathcal{M}_{\mathrm{red}}$ using $A^*$ search, with neighbor 
connectivity enforcing spatial proximity on $\mathcal{R}_{(n-1)R}$, 
periodic continuity on $S^1$, and branch consistency via $\det(J)$. 
Once a reduced path is found, the full $n$R trajectory is 
reconstructed by lifting each reduced state through the IK branch 
consistent with its stored $\det(J)$ label.

\begin{proposition}[Path Validity]
Any path found in $\mathcal{M}_{\mathrm{red}}$ under the determinant 
continuity constraint lifts to a valid continuous collision-free path 
in $\mathcal{C}^{(s)}_{\mathrm{free}} \subset T^n$.
\end{proposition}

\begin{proof}
Let $\gamma = (w_0, w_1, \ldots, w_k)$ be a path on the reduced 
lattice satisfying $|\det(J_{w_i}) - \det(J_{w_{i+1}})| < \varepsilon$ 
at every edge.

\textbf{Collision-free validity.} Each workspace sample $w_i$ was 
drawn from $\mathcal{C}^{(s)}_{\mathrm{free}}$ via $h$, so the 
recovered subchain configuration is collision-free by construction. 
Collision-free validity of each redundant link at stage $m$ is 
guaranteed by the reduced occupancy check at that stage. Therefore 
the full $n$R configuration at each step is collision-free.

\textbf{Kinematic continuity.} A discontinuous IK branch switch at 
step $i$ would require $\det(J)$ to change sign between $w_i$ and 
$w_{i+1}$. Such a sign change requires passage through zero, producing 
a jump of magnitude at least $|\det(J_{w_i})| + |\det(J_{w_{i+1}})|$. 
Away from $\Sigma$, this exceeds $\varepsilon$, violating the 
continuity constraint. Therefore no discontinuous branch switch occurs 
away from $\Sigma$. Near $\Sigma$ where $\det(J) \approx 0$, both 
terms are small and the constraint is satisfied across a sign change, 
correctly permitting smooth passage through a collision-free singular 
configuration.

Combining both claims, the lifted path is collision-free and 
kinematically continuous in $\mathcal{C}^{(s)}_{\mathrm{free}}$.
\end{proof}

\section{Experiments}
\label{sec:experiments}
\subsection{Experimental Setup}

We evaluate the proposed workspace-fibered decomposition framework on two classes of redundant planar manipulators (Figure~\ref{fig:example}). The first experiment considers a planar 3R manipulator as shown in Figure~\ref{fig:3R_example} operating in a workspace populated by circular obstacles. This system represents the minimal redundant planar manipulator for a position-only task. The manipulator consists of a 2R non-redundant positioning subchain with link lengths $ l_1 = l_2 = 1.0$ and a redundant terminal link with $l_3 = 0.6.$

The reduced planning space is constructed as
\[
\mathcal{M}_{\mathrm{red}}
=
\mathcal{R}_{2R}\times S^1
\]
where $\mathcal{R}_{2R}$ denotes the obstacle-constrained reachable workspace of the 2R subchain and $S^1$ parameterizes the orientation of the redundant terminal link.

The second experiment as shown in Figure~\ref{fig:5R_example} evaluates the scalability of the framework on a planar 5R manipulator with polygonal obstacles. The manipulator link lengths are
\[
(l_1,l_2,l_3,l_4,l_5)
=
(1.0,\;1.0,\;0.8,\;0.6,\;0.4).
\]

In contrast to the 3R case, the reduced space is constructed incrementally through successive fiber lifts:
\[
\mathcal{R}_{2R}
\rightarrow
\mathcal{R}_{2R}\times S^1
\rightarrow
\mathcal{R}_{3R}
\rightarrow
\mathcal{R}_{3R}\times S^1
\rightarrow
\mathcal{R}_{4R}
\rightarrow
\mathcal{R}_{4R}\times S^1.
\]

This experiment evaluates whether obstacle-constrained reachability information can be propagated recursively through successive redundant links without explicitly constructing the full configuration space $T^5$.

\subsection{Implementation Details}

\paragraph{2R Reachability Construction.}

For both experiments, the pipeline begins by constructing the configuration space of the 2R non-redundant subchain. The toroidal configuration space
\[
T^2 = S^1 \times S^1
\]
is discretized uniformly over
\[
\theta_1,\theta_2 \in [-\pi,\pi).
\]

Collision occupancy is computed through exact geometric intersection tests between manipulator links and workspace obstacles. Connected components on the toroidal lattice are extracted using a Union-Find structure with wraparound adjacency.

Starting from an initial configuration
\[
q_s = (\theta_1^s,\theta_2^s)
\]
the connected free-space component
\[
\mathcal{C}_{\mathrm{free}}^{(s)}
\subset T^2
\]
is identified and projected through the forward kinematics map
\[
h:T^2\rightarrow\mathbb{R}^2
\]
to obtain the obstacle-constrained reachable workspace
\[
\mathcal{R}_{2R}
\]
Each workspace sample additionally stores the Jacobian determinant
\[
\det(J)
=
l_1 l_2 \sin(\theta_2),
\]
which serves as a continuous local label distinguishing the two disconnected inverse-kinematic branches of the 2R subchain.

\paragraph{Incremental Fiber Lifting.}

For the 5R experiment, the reachable workspace is expanded incrementally by successively introducing redundant links.

Given a reachable workspace
\[
\mathcal{R}_{mR},
\]
the orientation of the next redundant link is discretized over an $S^1$ fiber:
\[
\mathcal{R}_{mR}\times S^1
\]

For each lifted state, only the newly-added link is collision-checked against the environment. Collision checking of the preceding subchain is omitted since collision-free validity is already guaranteed by membership in the reachable set of the previous stage.

The free subset of the lifted manifold is then projected through forward kinematics to obtain
\[
\mathcal{R}_{(m+1)R}
\]

This process is repeated recursively until the final reduced planning manifold is obtained.

\paragraph{Reduced-Space Planning.}

Planning is performed directly on the reduced manifold lattice using $A^*$ search.

Unlike a regular Euclidean grid, adjacency is not determined by lattice indexing. Instead, two reduced states are considered neighbors only if:
\begin{enumerate}
    \item Their workspace positions are spatially nearby,
    \item Their redundant fiber coordinates are locally continuous,
    \item The associated Jacobian determinants satisfy
    \[
    |\det(J_a)-\det(J_b)|<\epsilon.
    \]
\end{enumerate}

Importantly, the determinant field is always computed using only the Jacobian of the 2R non-redundant subchain, even for the higher-dimensional 5R system. This preserves consistency with the underlying inverse-kinematic branch decomposition induced by the base positioning chain.

The determinant continuity constraint prevents discontinuous transitions across disconnected self-motion branches and acts as a local topological consistency condition during graph construction.

\subsection{Results}

\paragraph{3R Workspace Structure.}

\begin{figure}[t!]
    \centering
    \begin{subfigure}{0.48\textwidth}
        \centering
        \includegraphics[width=1\linewidth]{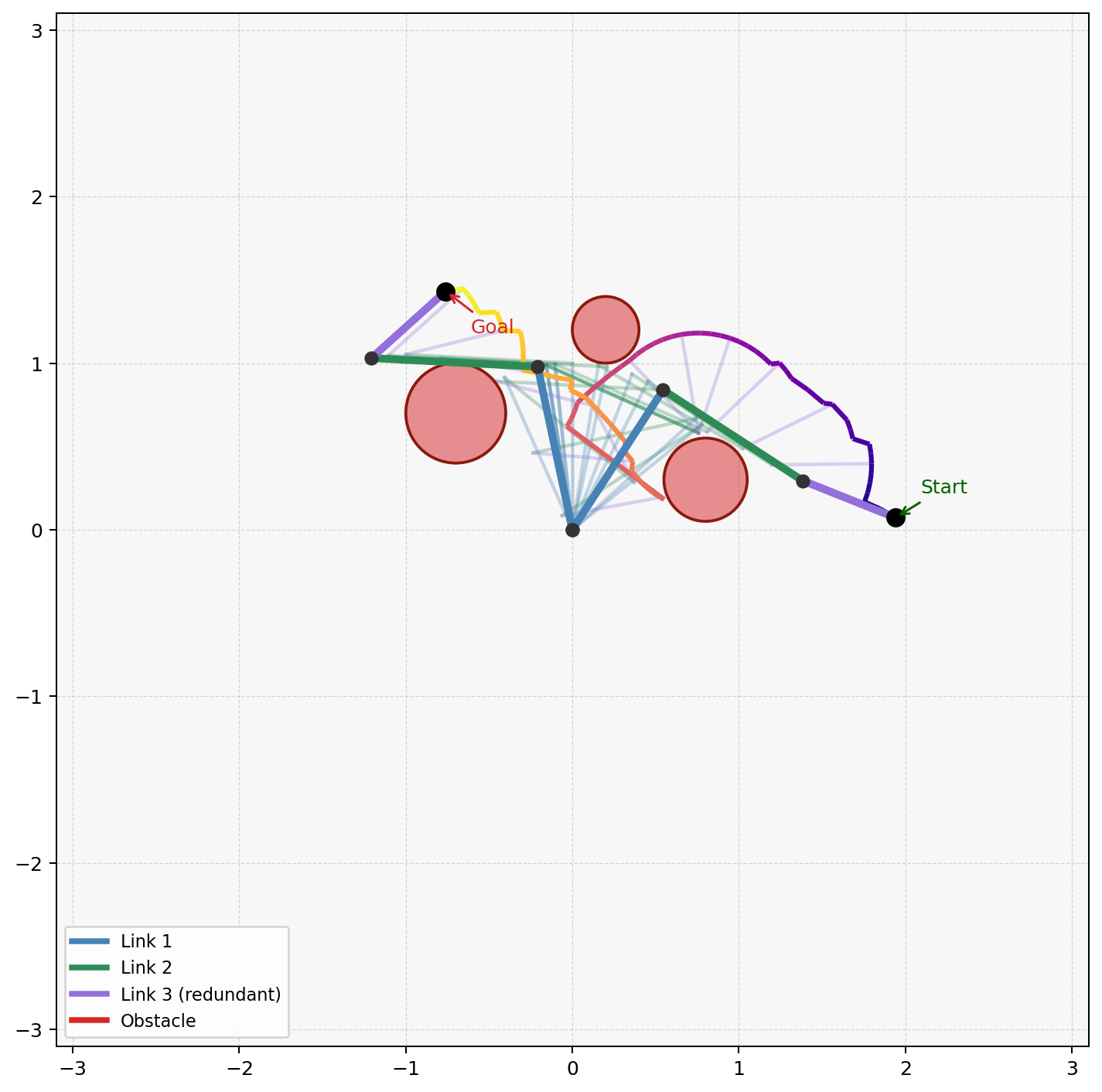}
        \caption{A-star planning for 3R manipulator}
        \label{fig:3R_example}
    \end{subfigure}
    ~
    \begin{subfigure}{0.48\textwidth}
        \centering
        \includegraphics[width=1\linewidth]{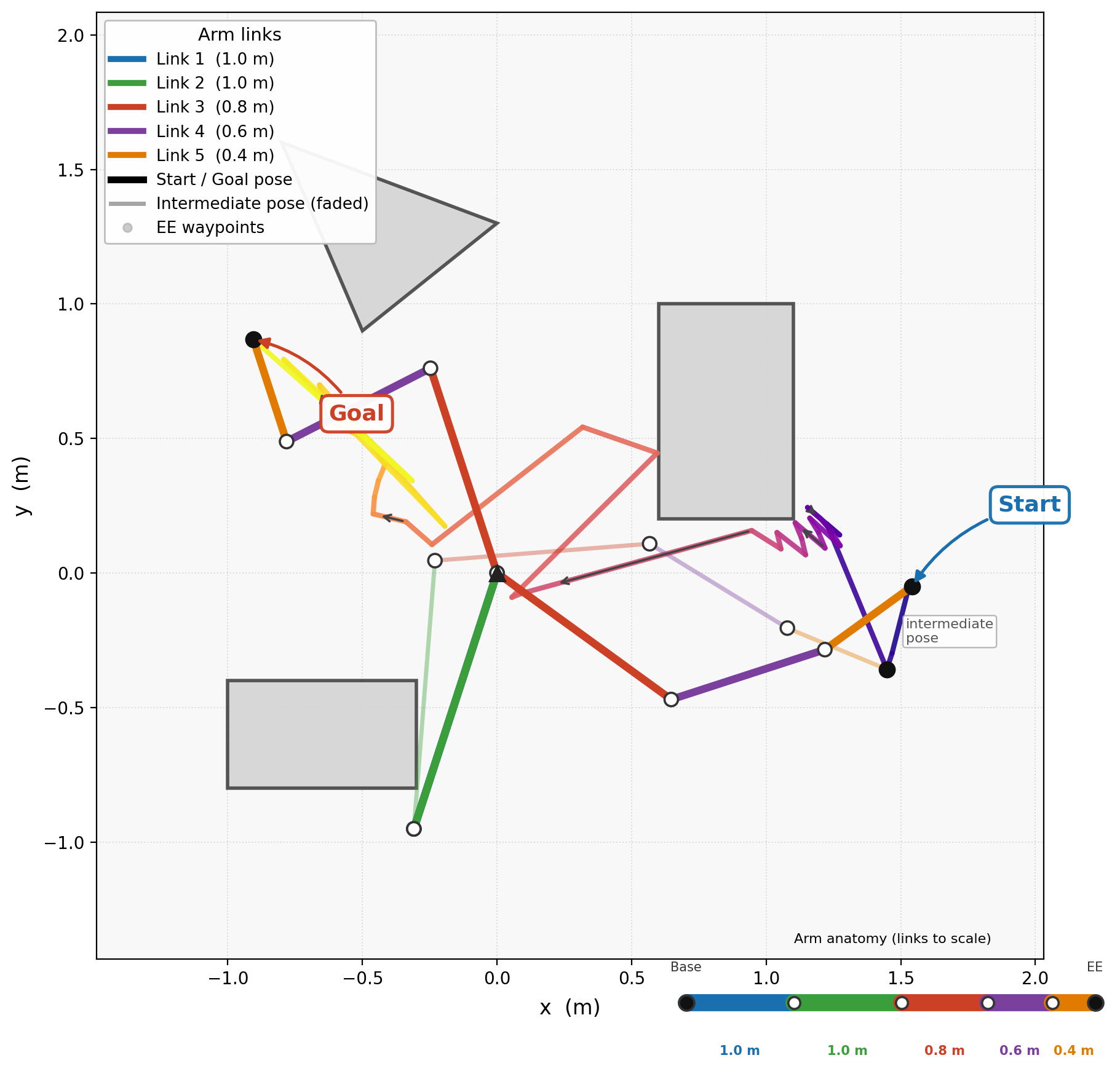}
        \caption{A-star planning for 5R manipulator}
        \label{fig:5R_example}
    \end{subfigure}
    \caption{Path planning using the reduced space}
    \label{fig:example}
\end{figure}

Figure~\ref{fig:workspace_detj} visualizes the reachable workspace $
\mathcal{R}_{2R} $
for the 3R experiment, colored by the Jacobian determinant field. The determinant field separates the workspace into two continuous regions corresponding to elbow-up and elbow-down inverse-kinematic branches. These regions meet only along the singular set
\[
\Sigma
=
\{q\in T^2 : \det(J)=0\}.
\]

Obstacle-induced occlusion partially disconnects the singular transition region, producing topological separation between reachable workspace branches without explicit configuration-space reasoning.

\paragraph{Reduced-Space Planning.}

The proposed planner successfully computes collision-free trajectories entirely within the reduced manifold representation. The lifted trajectories remain continuous in the original configuration space while avoiding discontinuous branch transitions.

In the 3R experiment, planning is performed on
\[
\mathcal{R}_{2R}\times S^1
\]
while in the 5R experiment planning occurs on the recursively-constructed manifold
\[
\mathcal{R}_{4R}\times S^1.
\]

The resulting trajectories demonstrate that obstacle-constrained reachability information can be propagated incrementally through multiple redundant links without constructing the full configuration-space obstacle geometry.

\paragraph{Incremental Reachability Propagation.}

Figure~\ref{fig:incremental_workspaces} shows the recursively-generated reachable workspaces
\[
\mathcal{R}_{2R},
\quad
\mathcal{R}_{3R},
\quad
\mathcal{R}_{4R},
\quad
\mathcal{R}_{5R}.
\]

The reachable regions expand geometrically as additional redundant links are introduced, while obstacle-induced exclusions propagate consistently through successive lifting stages.

These results suggest that the proposed decomposition behaves similarly to a recursive workspace extrusion process, where obstacle-constrained reachability information is transported through successive redundant fibers.

\paragraph{Computational Structure.}

The dominant computational cost arises during reduced manifold construction, where each lifted state requires collision evaluation only for the newly-added redundant link. Since the collision validity of preceding subchains is inherited from earlier stages, the approach avoids repeated full-chain collision checking throughout planning.

As a result, trajectory search itself operates entirely on the reduced manifold graph and is decoupled from explicit configuration-space obstacle geometry after preprocessing.

\begin{figure}[t]
    \centering
    \includegraphics[width=1\linewidth]{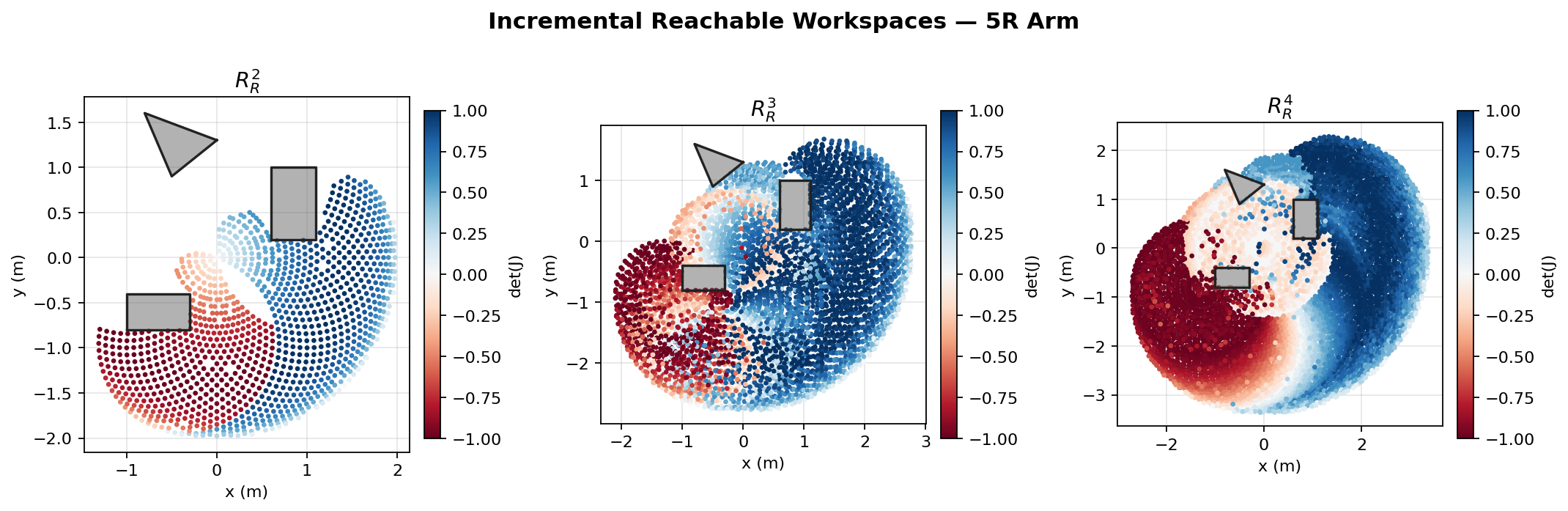}
    \caption{Incremental Workspaces for the 5R planar manipulator.}
    \label{fig:incremental_workspaces}
\end{figure}

\section{Conclusion}
\label{sec:conclusion}
This work introduced a workspace‑fibered decomposition framework for complete motion planning in redundant planar manipulators, building the reachable planning domain incrementally, one redundant degree of freedom at a time. Starting from the minimal non‑redundant 2R positioning subchain, we construct an obstacle constrained reachable workspace that faithfully inherits the topological structure of the connected free component in configuration space. Successive lifting of orientation fibers over this base, combined with projection through forward kinematics, yielded a hierarchy of reduced planning spaces $\mathcal{R}_{mR}\times S^1$. that propagate reachability information without explicitly constructing the full configuration space.

A key ingredient of the framework is the use of the 2R Jacobian determinant as a continuous, task‑aligned branch label that encodes the elbow‑up and elbow‑down inverse‑kinematic branches. Enforcing a local determinant continuity constraint in the reduced lattice suppresses discontinuous branch switches, while still permitting smooth traversal of collision‑free singular configurations. This yields a branch‑consistent planner that operates entirely on the reduced manifold, yet reconstructs continuous, collision‑free trajectories in the original configuration space via incremental lifting.

The experimental evaluation on 3R and 5R planar manipulators demonstrated that the proposed decomposition can propagate obstacle‑constrained reachability through multiple redundant links using only local collision checks on newly introduced links at each stage. In both cases, planning on the final reduced manifold produced lifted trajectories that respected the topology of the underlying self‑motion structure and avoided explicit configuration‑space obstacle reasoning during search. These results suggest that workspace fibration provides a viable alternative to full‑space planning for a class of redundant arms, with a favorable computational structure that decouples preprocessing from trajectory search.

Beyond the specific planar case studies considered here, the framework offers a geometric template for combining task‑induced foliations with fiber‑type factorizations of configuration space, complementing recent work on quasi orthogonal foliations and position‑level redundancy resolution. 

Future work will extend the analysis to spatial manipulators and higher‑ dimensional tasks, investigate stronger theoretical guarantees on the underlying transverse distributions and holonomy, and develop quantitative comparisons against state‑of ‑the‑art configuration space and manifold‑based planners on more complex benchmark environments.

\bibliographystyle{styles/bibtex/splncs03_unsrt}
\bibliography{ref}

\end{document}